\documentclass[letterpaper]{article} 
\usepackage{aaai23}  
\usepackage{times}  
\usepackage{helvet}  
\usepackage{courier}  
\usepackage[hyphens]{url}  
\usepackage{graphicx} 
\usepackage{natbib}  
\usepackage{caption} 
\usepackage[utf8]{inputenc}
\usepackage{amsmath,amssymb}
\usepackage{amsthm}
\usepackage{mathtools} 
\usepackage{stmaryrd}
\usepackage{tikz}
\usepackage[ruled, linesnumbered, vlined]{algorithm2e}
\usepackage{booktabs}

\usetikzlibrary{automata,angles,quotes}

\tikzstyle{tran}=[thick,draw,->,>=stealth,rounded corners]
\tikzstyle{state}=[circle,thick,draw,minimum size=1.5em,%
inner sep=0em]

\newtheorem{example}{Example}
\newtheorem{theorem}{Theorem}
\newtheorem{lemma}{Lemma}
\newtheorem{definition}{Definition}

\newcommand{\pomdp}{P}
\newcommand{\cpomdp}{\mathcal{C}}
\newcommand{\states}{S}
\newcommand{\act}{{A}}
\newcommand{\trans}{\delta}
\newcommand{\obs}{\mathcal{Z}}

\newcommand{\obsmap}{\mathcal{O}}
\newcommand{\distr}{\mathcal{D}}
\newcommand{\cons}{C}
\newcommand{\obsfunc}{\mathcal{O}}
\newcommand{\Rset}{\mathbb{R}}

\newcommand{\Nset}{\mathbb{N}}

\newcommand{\reloads}{R}
\newcommand{\ob}{o}
\newcommand{\Ob}{O}
\newcommand{\mysucc}{\mathit{Succ}}
\newcommand{\supp}{\mathit{supp}}
\newcommand{\eps}{\varepsilon}
\newcommand{\ca}{\mathit{ca}}
\newcommand{\tstep}{t}
\newcommand{\reslevrv}{L}
\newcommand{\initdistr}{\lambda_0}
\newcommand{\initlev}{\ell_0}
\newcommand{\probm}{\mathbb{P}}
\newcommand{\expv}{\mathbb{E}}

\newcommand{\tgt}{G}
\newcommand{\cost}{\mathit{cost}}

\newcommand{\tcost}{\textit{TC}}
\newcommand{\shield}{\sigma}
\newcommand{\en}{1}
\newcommand{\dis}{0}
\newcommand{\hist}{\textit{Hist}}
\newcommand{\nint}[1]{[#1]}
\newcommand{\tlev}{\mathit{TLev}_{=1}}
\newcommand{\tlevPR}{\mathit{TLev}_{>0}}
\newcommand{\belsupf}{\mathit{Bel}}
\newcommand{\len}{\mathit{len}}
\newcommand{\belsupup}{\Gamma}

\newcommand{\resup}{\Delta}
\newcommand{\resupinv}{\Psi}

\newcommand{\bels}{B}
\newcommand{\tok}[1]{{#1}_{T}}
\newcommand{\guess}{\alpha}
\newcommand{\eguess}{\varepsilon}
\newcommand{\obc}{\mathit{oc}}
\newcommand{\bsucc}{\mysucc}

\newcommand{\toremove}{\mathit{Rem}}

\title{Shielding in Resource-Constrained Goal POMDPs}

\author{
	Michal Ajdar\'{o}w, \v{S}imon Brlej, Petr Novotn\'{y}
}
\affiliations{
	Faculty of Informatics, Masaryk University\\
	\{xajdarow,xbrlej,petr.novotny\}@fi.muni.cz
}

\begin{document}

\maketitle
\begin{abstract}
We consider partially observable Markov decision processes (POMDPs) modeling an agent that needs a supply of a certain resource (e.g., electricity stored in batteries) to operate correctly. The resource is consumed by agent's actions and can be replenished only in certain states. The agent aims to minimize the expected cost of reaching some goal while preventing resource exhaustion, a problem we call \emph{resource-constrained goal optimization} (RSGO). We take a two-step approach to the RSGO problem. First, using formal methods techniques, we design an algorithm computing a \emph{shield} for a given scenario: a procedure that observes the agent and prevents it from using actions that might eventually lead to resource exhaustion. Second, we augment the POMCP heuristic search algorithm for POMDP planning with our shields to obtain an algorithm solving the RSGO problem. We implement our algorithm and present experiments showing its applicability to benchmarks from the literature.
\end{abstract}


\section{Introduction}

Partially observable Markov decision processes (POMDPs) are the standard model for decision making under uncertainty. While POMDPs are computationally demanding to solve, advances in heuristic search~\cite{SV:2010:POMCP} and reinforcement learning~\cite{BhaBWGB:2020:pomdp-alphazero} allowed for tackling large POMDP models. Recently, an increasing attention is being paid to \emph{safety aspects} of autonomous decision making as opposed to pure optimization of the expected rewards or costs~\cite{GarF:2015:saferl}. Indeed, heuristic and learning techniques can be susceptible to leading the decision-making agent into dangerous situations, and additional care must be taken to \emph{formally guarantee} the absence of such a risky behavior.

\paragraph{Shielding.} A promising approach to obtaining such safety guarantees is offered by the concept of \emph{permissive controller synthesis}~\cite{DraFKPU:2015-permissive-synth-journal, JunJDTKL2016:safe-rl-permissive}, which was later distilled into the concept of \emph{shielding}~\cite{AlsBEKNT:2018:shielding}. Intuitively, a shield \( \shield \) is a function which inputs the agent's current information \( X \) (in a POMDP, this would be the whole history of the agent's actions and observations) and outputs a list of \emph{allowed} actions which the agent can use in the current situation without risking violation of its safety specification in the future.
The shielding process then forces the agent to only consider, in a situation described by available information \( X \), the allowed actions in the list \( \shield(X) \). 
Shields are typically computed via \emph{formal methods} approaches, and hence they can guarantee that the shielded algorithm satisfies the desired safety specification.

\paragraph{Consumption Models and RSGO Problem.} In this paper, we focus on safe decision making in agents that require some uninterrupted supply of a resource (such as electricity) to operate correctly. Such agents can be encountered, e.g., in robotics~\cite{NotRE:2018:persistification}, and we model them via the \emph{consumption} paradigm~\cite{BCKN:consumption-games}, where the available amount of the resource is represented by an integer from the set \( \{0,1,\ldots,\ca\} \), the number \( \ca \) denoting the agent's \emph{battery capacity}. Each action of the agent consumes some amount of the resource (i.e., decreases the resource level) and the resource can be replenished to the full capacity only in special \emph{reload states} (e.g., charging stations). The safety specification is that the agent must never run out of the resource. A crucial property of consumption models (such as \emph{consumption MDPs}~\cite{BlaBNOTT:2020:cons-mdp-buchi}) is that the resource levels are \emph{not} encoded in the model's states (since this 
 would blow up the state space by a factor exponential in the bit-size of \( \ca \)). Instead, the amount of resource consumed is specified for each state-action pair, and the agent then tracks its resource level itself. To account for the fact that, apart from preventing resource exhaustion, the agent aims to do something useful, we study the \emph{resource-safe goal optimization (RSGO)} problem in the consumption setting: the agent aims to reach a given set of \emph{goal states} with probability 1 (i.e., almost-surely) at the minimal expected cost, while preventing resource exhaustion.

\paragraph{Limitations of Previous Work.} The previous approaches to consumption models for resource-constrained decision making under uncertainty suffer from two key limitations:
First, they consider only perfectly observable setting, i.e., consumption MDPs.
Second, they only consider computing policies satisfying \emph{qualitative criteria:} avoiding resource exhaustion and almost-surely reaching a goal state; optimization of \emph{quantitative criteria,} such as the expected cost of achieving a goal, was not considered.
\paragraph{Our Contribution.}
In this paper, we overcome both aforementioned limitations: we present a method for solving a combination of qualitative and quantitative criteria in \emph{partially observable} consumption MDPs (CoPOMDPs). Our contribution has two essential parts: First, we show how to design an algorithm computing shields in consumption POMDPs prohibiting exactly those behaviors that lead to resource exhaustion or that violate the possibility of eventually reaching a goal state. Hence, our shields handle the qualitative aspect of the RSGO problem. Second, to handle also the quantitative aspect, we augment the well-known POMCP heuristic planning algorithm~\cite{SV:2010:POMCP} with our shields, thus obtaining an algorithm for the finite-horizon approximation of the RSGO problem. We implement our new algorithm and demonstrate its applicability to benchmarks derived from the literature. 

\paragraph{Outline of Techniques.} The previous work~\cite{BlaBNOTT:2020:cons-mdp-buchi} presented an algorithm that for \emph{perfectly observable} consumption MDPs (CoMDPs) computes a policy ensuring almost-sure goal reachability while preventing resource exhaustion. The algorithm runs in time polynomial in size of the CoMDP, so in particular in time polynomial in the bit-size of \( \ca \). We reduce the computation of shields for CoPOMDPs to the CoMDP problem in~\cite{BlaBNOTT:2020:cons-mdp-buchi}. The reduction is non-trivial: in particular, we show that the standard technique of constructing a \emph{belief support MDP}~\cite{ChaCGK:2016:POMDP-belsup-goalcost,JunJS:2021:pomdp-asreach-SMT, BaiBG:2008:POMDP-asreach-dec} and then applying the algorithm for perfectly observable MDPs is not directly usable in the consumption setting. 

\paragraph{Related Work.}

Shielding  in MDPs and POMDPs was studied w.r.t. state safety specification (avoiding critical states)~\cite{AlsBEKNT:2018:shielding}, ensuring almost-sure reachability of a goal state~\cite{JunJS:2021:pomdp-asreach-SMT}, or guaranteeing that the payoff is almost-surely above some threshold~\cite{ChaNPRZ:2017:mcts-safety}. The related notion of permissive controller synthesis has been studied in more quantitative settings of probabilistic reachability and expected cost~\cite{DraFKPU:2015-permissive-synth-journal, JunJDTKL2016:safe-rl-permissive}. To our best knowledge, our paper is the first to consider shielding for resource-constrained agents modeled via POMDPs. While shielding for CoPOMDPs could be in principle reduced to state safety shielding by encoding resource level into states, this would blow-up the state space by the factor equal to the battery capacity. As demonstrated already for CoMDPs in~\cite{BlaBNOTT:2020:cons-mdp-buchi}, the ``resource levels in states'' approach is highly inefficient when compared to methods tailored to consumption models. 

Resource-constrained planning is particularly relevant in the application domain of autonomous (land-based/aerial/underwater) vehicles, e.g.,~\cite{NotRE:2018:persistification,MitCCSM:2015:control-fuel-robots,EatKCM:2018:fuel-pomdp-robot}. In formal methods and verification, resource-constrained agents are typically modeled in the consumption framework (used in this paper) or \emph{energy} framework. The former has been considered in both non-probabilistic~\cite{BCKN:consumption-games,BKKN:consumption-payoff} and probabilistic~\cite{BlaBNOTT:2020:cons-mdp-buchi,BlaCNOTT:2021:FiMDP} settings, but not in partially observable ones. The energy framework differs from the consumption one by handling reloads: instead of atomic reloads in reload states, a ``negative consumption'' is enabled for some state-action pairs, allowing for incremental reloading. Various classes of energy models have been considered, e.g.,~\cite{CdAHS:resource-interfaces,BFLMS:weighted-automata-inf-runs,BKN:energy-mp-ATVA,MSTW:energy-parity-MDPs,DDGRT10:PO-energy-MP}, and while they can be seen as more general than consumption models, they are not known to admit algorithms running in time polynomial in the bit-size of the capacity \( \ca \).

The constrained optimization aspect present in CoPOMDPs is similar in spirit to \emph{constrained (PO)MDPs} (C(PO)MDPs)~\cite{Altman:1999:CMDP-book, UndH10:constrained-pomdp-online, PouMPKGB15:constrained-POMDP}. While both approaches fit the same framework of constrained optimization, the types of constraints are actually quite different. In particular, Co(PO)MDPs \emph{cannot} be viewed as a special case of C(PO)MDPs, and vice versa. The key difference is that in C(PO)MDPs, there are penalty functions determining one-step penalties, and the constraint is that the \emph{expected aggregated} penalty is below a given threshold. (Aggregation functions such as discounted, total, or mean payoff are typically considered.) On the other hand, the constraint in Co(PO)MDPs is on the \emph{intermediate} values of the consumption, not on the expectation of its aggregate values. C(PO)MDPs can impose constraints such as “the expected total amount of the resource consumed by the agent is \( \leq B \),” which does not guarantee that the agent does not run out of the resource between two reloads. Hence, the two models are incomparable.

\section{Preliminaries}

We denote by \( \distr(X) \) the set of all probability distributions over an at most countable set \( X \) and by \( \supp(d) \) the support of a distribution \( d \). For \( n\in \Nset \) we denote by \( \nint{n} \) the integer interval \( \{0,1,\ldots,n\} \cup\{\bot\} \), where \( \bot \) is a special element 
deemed smaller than all integers.

\begin{definition}[CoPOMDPs.]
	A \emph{consumption partially observable Markov decision process (CoPOMDP)} is a
	tuple $\cpomdp=(\states,\act,\trans,\obs,\obsmap,\cons,\reloads,\ca)$ 
	where
	$\states$ is a finite set of \emph{states},
	$\act$ is a finite set of \emph{actions},
	$\trans:\states\times\act \rightarrow \distr(\states)$ is a 
	\emph{probabilistic transition function} that given a state $s$ and an
	action $a \in \act$ gives the probability distribution over the successor 
	states, 
	$\obs$ is a finite set of \emph{observations}, and
	$\obsmap:\states\rightarrow \distr(\obs)$ is a probabilistic 
	\emph{observation function} that 
	maps every state to a distribution over observations, 
	\( \cons \colon \states \times \act \rightarrow \Nset \)
	is a \emph{resource consumption function,} \( \reloads \subseteq \states \) is the set of \emph{reload states,} and \( \ca \in \Nset^+ \) is a \emph{resource capacity.}
%
%
	
	We often abbreviate $\trans(s,a)(s')$ and $\obsfunc(s)(\ob)$ by 
	$\trans(s'|s,a)$ and $\obsfunc(\ob|s)$, respectively. 
\end{definition}

For \( (s,a)\in \states \times \act \) we denote by \( \mysucc(s,a) \) the set of successor states of \( s \) under \( a \), i.e., the support of \( \trans(s,a) \). 
We say that two states \( s,t\in \states \) are \emph{lookalikes} if they can produce the same observation, i.e., if \( \supp(\obsfunc(s))\cap \supp(\obsfunc(t)) \neq \emptyset\). 

\paragraph{Dynamics of CoPOMDPS.} A CoPOMDP \( \cpomdp \) evolves in discrete time steps. The situation at time \( \tstep \in \Nset \) is described by a pair of random variables, \( S_t \) and \( \reslevrv_t \), denoting the current state and the current resource level at time \( t \), respectively. The agent cannot observe the state directly, instead receiving an \emph{observation} \( \Ob_t \in \obs \) sampled according to the current state and \( \obsfunc \): \( \Ob_t \sim \obsfunc(S_t) \). The initial state is given by the \emph{initial distribution} \( \initdistr \in \distr(\states) \), i.e., \( S_0 \sim \initdistr \), while \( \reslevrv_0 \) will be typically fixed to a concrete \emph{initial resource level} \( \initlev \in\nint{\ca} \). 
Then, in every step \( \tstep \), the agent selects an action \( A_t \) from \( \act \) according to a \emph{policy} \( \pi \). The policy makes a (possibly randomized) decision based on the 
current \emph{history} \( H_t = \Ob_0 \reslevrv_0 A_0 \Ob_1\reslevrv_1 A_1\ldots \Ob_{t-1} \reslevrv_{t-1} A_{t-1}\Ob_{t}\reslevrv_t\), a finite alternating sequence of hitherto witnessed observations, resource levels, and actions; i.e., \( A_t \sim \pi(H_t) \). The new state \( S_{t+1} \) is then sampled according to the transition function \( \trans \): \( S_{t+1} \sim \delta(S_t,A_t) \).
To describe the resource dynamics, we define a \emph{resource update function} \( \resup \colon\nint{\ca}\times \states\times\act \rightarrow \nint{\ca} \) s.t. \( \resup(\ell,s,a) \) denotes the resource level after making a step from \( s \) using action \( a \), provided that the previous resource level was \( \ell \): if \( \ell= \bot \), then  \( \resup(\ell,s,a) = \bot\) and otherwise
\[
\resup(\ell,s,a) = \begin{cases}
\ell-\cons(s,a) & \text{if }  0 \leq \ell-\cons(s,a) \wedge s  \not\in \reloads \\
\ca - \cons(s,a) & \text{if } 0 \leq \ca-\cons(s,a) \wedge s \in \reloads \\
\bot & \text{otherwise}.
\end{cases}
\]

\noindent
We then put \( \reslevrv_{t+1} = \resup(\reslevrv_t,S_t,A_t) \). 

When referring to CoPOMDP dynamics, we use upper-case to denote random variables (\( S_t,\Ob_t,A_t \), etc.) and lower case for concrete values of variables, e.g., \( h_t \) for a concrete history of length \( t \). We denote by \( \hist \) the set of all possible histories in a given CoPOMDP, and by \( \len(h) \) a \emph{length} of a history \(h\), i.e., the number of action appearances in \( h \). We denote by \( \ell_h \) the last resource level of history \( h \).

%

\paragraph{Assumptions: Observable Resource and Zero Cycles.} Including resource levels in histories amounts to making the levels perfectly observable. This is a reasonable assumption, as we expect a resource-constrained agent to be equipped with a resource load sensor. As long as this sensor has at least some guaranteed accuracy \( \eps \), for an observed resource level \( L \) we can report \( \lfloor L - \eps \rfloor \) to the agent as a conservative estimate of the resource amount. We also assume that the agent can recognize whether it is in a reload state or not, i.e., a reload state is never a lookalike of a non-reload state. Finally, we assume that the CoPOMDP does not allow the agent to indefinitely postpone consuming a positive amount of a resource unless a goal state (see below) has been already reached. This is a technical assumption which does not impede applicability, since we typically model autonomous agents whose \emph{every} action consumes some resource.

\paragraph{Optimization in Goal CoPOMDPs.}

In what follows, we denote by \( \probm^\pi \) the probability measure induced over the trajectories of a CoPOMDP by a policy \( \pi \), and by \( \expv^\pi \) the corresponding expectation operator.

In a \emph{goal CoPOMDP,} we are given a set of \emph{goal states} \( \tgt \subseteq \states \). A policy \( \pi \) is a \emph{goal} policy if it reaches a goal state almost-surely, i.e., with probability 1:
\[ \probm^{\pi}(S_t \in \tgt \text { for some } t \in \Nset ) = 1, \]
and a \emph{positive-goal} policy if there exists \(\epsilon>0 \) such that it reaches a goal state with probability at least \(\epsilon \), i.e.:
\[ \probm^{\pi}(S_t \in \tgt \text { for some } t \in \Nset ) \geq \epsilon. \]
We assume that all goal states \( g \in \tgt \) are absorbing, i.e., \( \trans(g\mid g,a) = 1 \) for all \( a \in \act \); that the self-loop on \( g \) has a zero consumption; and that the agent can observe reaching a goal, i.e., no goal state is lookalike with a non-goal state. This captures the fact that reaching a goal ends the agent's interaction with the environment.

We study the problem of reaching a goal in a most efficient way while preventing resource exhaustion. To this end, we augment the CoPOMDP with a \emph{cost} function \( \cost \colon \states\times\act\rightarrow \Rset_{\geq 0} \), stipulating that the self-loops on goal states have zero costs (we can also allow negative costs for some state-action pairs,  as long as such pairs can appear only finitely often on each trajectory: these can be used, e.g., as one-time ``rewards'' for the agent reaching a goal state). Then the \emph{total cost} of a policy \( \pi \) is the quantity 
\[\tcost(\pi) = \expv^\pi\left[\sum_{t=0}^{\infty} \cost(S_t,A_t)\right].\] 


\paragraph{Resource-Constrained Goal Optimization.}

A policy \( \pi \) is \emph{safe} if it ensures that the resource is never exhausted; due to the discrete nature of CoPOMDPs, this is equivalent to requiring that the exhaustion probability is zero:
\[
\probm^\pi(\reslevrv_t = \bot \text{ for some } t\in \Nset) = 0.
\]

\noindent
In the \emph{resource-constrained goal optimization (RSGO)} problem we aim to find a policy \( \pi \) minimizing \( \tcost(\pi) \) subject to the constrain that \( \pi \) is a \emph{safe goal} policy.

\paragraph{Belief Supports for CoPOMDPs.}
\label{subsec:belsup}

When working with POMDPs, one often does not work directly with histories but with their suitable \emph{statistics}. 
The \emph{constraints} in the RSGO problem are \emph{qualitative}, and hence \emph{qualitative} statistics should be sufficient for satisfying resource safety. This motivates the use of \emph{belief supports}. 
For each history \( h \in \hist \), the belief support of \( h \) is the set \( \belsupf(h) \subseteq \states\) of all states in which the agent can be with a positive probability after observing history \( h \). 
The standard formal definition of a belief support~\cite{ChaCGK:2016:POMDP-belsup-goalcost}  needs to be extended to incorporate resource levels.
 This extension is somewhat technical and we defer it to the supplementary material. Similarly to standard belief supports, it holds that for a history \( h = \hat{h}a\ob\ell \) with prefix \( \hat{h} \), given  \( \belsupf(\hat{h}), a ,\ob, \ell  \), and last resource level of \( \hat{h} \), we can compute \( \belsupf(h) \) in quadratic time. Thus, the agent can efficiently update its belief support when making a step.

Given a belief support \( B \), state \( s \in B \), and any history \( \hat{h}a\ob\ell \) s.t. \( B = \belsupf{(\hat{h})} \) and \( \obsfunc(\ob| t)>0 \) for some \( t \in \mysucc(s, a) \), we say that the belief support \( \belsupf(\hat{h}a\ob\ell) \) is the \( s \)-successor of \( B \) under \( a \).  Slightly abusing the notation, we denote by \( \bsucc(\bels, a, s) \) the set of all possible \( s \)-successors of \( B \) under \( a \). We also denote \( \bsucc(\bels, a) = \bigcup_{s \in B} \bsucc(\bels, a, s) \). Given \( B \) and \( a \) (or \( s \)), \( \bsucc(\bels,a) \) and \( \bsucc(\bels,a,s) \) can be computed in polynomial time.

\section{Shielding for CoPOMDPs}



Informally, a shield is an algorithm which, in each step, \emph{disables} some actions (based on the available information) and thus prevents the agent from using them. In this paper, we formalize shields as follows:


\begin{definition}
A \emph{shield} is a function \( \shield \colon \hist\times \act \rightarrow\{\en,\dis\}\). We say that \( \shield \) \emph{enables} action \( a\in \act \)  in history \( h \) 
if \( \shield(h,a)=1 \), otherwise \( \shield \) \emph{disables} \( a \).

We say that policy \( \pi \) \emph{conforms} to \( \shield \) if it never selects a disabled action, i.e., if  \( \shield(h,a)=0 \) implies \( \pi(h)(a)=0 \), for all \( h\in \hist,a\in\act \).

A shield \( \shield \) is support-based if \( \shield(h,a) = \shield(h',a) \) for any action \(a\) and histories \( h,h' \) s.t. \( \belsupf(h)=\belsupf(h') \) and \( \ell_h = \ell_{h'} \). We treat support-based shields as objects of type \( 2^\states\times \nint{\ca}\times\act\rightarrow \{\en,\dis\}. \) A support-based shield \( \shield \) is \emph{succinct} if for every \( B\in 2^\states \) and every \( a \in \act \) there is a threshold \( \tau_{B,a} \) (possibly equal to \( \infty \)) such that \( \sigma(B, \ell, a) = 1 \Leftrightarrow \ell \geq \tau_{B,a}.\)
\end{definition}

Succinct shield can be represented by a table of size \( \mathcal{O}(2^{|\states|}\cdot |\act|) \) storing the values \(\tau_{B,a}\). Hence, we treat them as functions of the type \( 2^\states \times \act \rightarrow \nint{\ca} \cup\{\infty\}\).


\paragraph{RSGO Problem and Shields.}

We aim to construct shields taking care of the qualitative constraints in the RSGO problem.
To this end, our shields need to prevent two types of events: directly exhausting the resource, and getting into situation where a goal state cannot be reached almost surely without risking resource exhaustion. The latter condition can be formalized via the notion of a trap:

\begin{definition}
A tuple \( (\bels,\ell) \), where \( \bels\subseteq 2^\states \), is a \emph{trap} if setting the initial distribution \( \initdistr \) to the uniform distribution over \( \bels \) and the initial resource level \( L_0 \) to \( \ell \) yields a CoPOMDP in which no safe goal policy exists.
\end{definition}

Moreover, an ideal shield should not over-restrict the agent, i.e., it should allow any behavior that does not produce some of the two events above. The following definition summarizes our requirement on shields.

\begin{definition}
A shield \( \shield \) is \emph{exact} if it has the following three properties:
\begin{itemize}
\item every policy conforming to \( \shield \) is safe; and
\item each policy \( \pi \) conforming to \( \shield \) avoids traps, i.e., satisfies 
\[
\probm^\pi((\belsupf(H_t),\reslevrv_t) \text{ is a trap for some } t\in \Nset) = 0\text{; and}
\]
\item \emph{every} safe goal policy conforms to \( \shield \).
\end{itemize}
\end{definition} 

In Section~\ref{sec:fipomdp}, we show how to employ heuristic search augmented by an exact shield to solve a finite-horizon approximation of the RSGO problem. Before that, we present an algorithm which, given a CoPOMDP \( \cpomdp \), computes an exact shield  for \( \cpomdp \) and checks whether the RSGO problem for \( \cpomdp \) has a feasible solution.



\section{Computing Exact Shields}
We construct exact shields by computing \emph{threshold levels.} These indicate, for each situation, what is the minimal resource level with which the agent can still manage to almost-surely reach a goal state without exhausting the resource. 

\begin{definition}
\label{def:thr-level}
Let \( \cpomdp \) be a CoPOMDP and \( h \in \hist \) be its history. The \emph{threshold level} and \emph{positive-threshold level} of \( h \) in \( \cpomdp \) are the quantities \( \tlev^{\cpomdp}(h) \) and \(\tlevPR^{\cpomdp}(h) \), respectively, equal to the smallest resource level \( \ell \in \nint{\ca} \) that has the following property: if we set the initial distribution \( \initdistr \) of \( \cpomdp \) to the uniform distribution over the set \( \belsupf(h) \) and the initial resource level \( \reslevrv_0 \) to \( \ell \), then the resulting CoPOMDP has a:
\begin{itemize}
\item safe goal policy in the case of \( \tlev^\cpomdp(h) \); 
\item safe positive-goal policy in the case of \( \tlevPR^\cpomdp(h) \).
\end{itemize} 
In the case such a resource level does not exist at all, we set the respective value to \( \infty \). We omit the superscript if \( \cpomdp \) is clear from the context.
\end{definition}

The next lemma connects threshold levels to shields.\footnote{Full proofs, missing constructions, and benchmark details are in the supplementary material.}

\begin{lemma}
\label{lem:thr-level-shielding}
Let \( \shield \) be an exact shield. Then for each \( h\in \hist \) and each \( a \in \act \) we have that \( \shield(h,a) = \en \) if and only if \( \ell_h \) is greater than or equal to the smallest number \( \tau \) s.t. for any \( s \in \belsupf(h) \) and any valid history of the form \( ha\ob\ell \) s.t. \( \obsfunc(\ob | t)>0  \) for some \( t \in \mysucc(s, a) \), it holds that \( \resup(\tau, s, a) \geq \tlev(ha\ob\ell) \).
\end{lemma}
\noindent
I.e., an action can be enabled iff all possible outcomes of that action lead to a situation where the current resource level is at least the threshold level for that situation. Note that Lemma~\ref{lem:thr-level-shielding} entails the existence of a unique (up to values in histories that already exhausted the resource) exact shield. Moreover, since threshold levels of a history \( h \) only depend on \( \belsupf(h) \), this exact shield is support-based and succinct.

\subsection{Computing Threshold Levels in POMDPs}

We build on the fact, established in the previous work, that threshold levels can be efficiently computed for \emph{perfectly observable} consumption Markov decision processes, or CoMDPs. In CoMDPs, the agent can perfectly observe the current state \( S_t \) and the history of past states. For notational convenience, we treat CoMDPs as a special case of CoPOMDPs in which \( \obs = \states \) and each state always produces itself as its unique observation: we thus omit the observations and observation function from the description of a CoMDP.

\begin{theorem}[\cite{BlaBNOTT:2020:cons-mdp-buchi}]
\label{thm:cmdp-tvals-algo}
In a CoMDP, the values \( \tlev(h) \) and \(\tlevPR(h)\) depend only on the last observation (i.e., state) of \( h \). Moreover, for each state \( s \), the values \( \tlev(s) \) and \(\tlevPR(s)\) -- the common values of \( \tlev(h) \) and \(\tlevPR(h)\), respectively, for all \( h\)'s whose last state equals to \( s \) --  can be computed in polynomial time.
\end{theorem}

The main idea of our approach is to turn a given CoPOMDP \( \pomdp \) into a \emph{finite} CoMDP that captures the aspects of CoPOMDP dynamics pertaining to threshold levels. Below, we present a construction of such a CoMDP.

\paragraph{Consumption Consistency.} 

 The construction assumes that the input CoPOMDP is \( \emph{consistent}\), i.e., \( \cons(s,a)=\cons(t,a) \) for each pair of lookalike states \( s,t \). Any CoPOMDP can be easily transformed into consistent one by splitting each probabilistic transition \( \delta(s,a)(t) \) with a dummy state \( t_{s,a} \) in which the consumption depending on \( s\) and \( a \) takes place. (The state \( t_{s,a} \) emits the value \( \cons(s,a) \) as its observation, which only gives the agent information that he is guaranteed to get in the next step anyway). 

\begin{lemma}
Given a CoPOMDP \( \cpomdp \) one can construct, in time linear in the size of \( \cpomdp \), an equivalent (in terms of policies and their costs) consistent CoPOMDP \( \cpomdp' \). 
\end{lemma}

\paragraph{Token CoMDPs.}
A \emph{token CoMDP} is our generalization of a ``POMDP to MDP'' construction used in~\cite{BaiBG:2008:POMDP-asreach-dec} to prove decidability of almost-sure reachability in (standard) POMDPs. States of the token CoMDP correspond to tuples \( (B,\alpha) \), where \( B \) is a belief support in the original CoPOMDP and \( \alpha \in B \) is a ``token'', signifying the agent's guess of the current state (\( \alpha \) can also equal a special symbol \( \varepsilon \), representing an invalid guess). Formally, given a consistent CoPOMDP \( \cpomdp = (\states,\act,\trans,\obs,\obsmap,\cons,\reloads,\ca)\) we construct a \emph{token} CoMDP \( \tok\cpomdp = (\tok\states,\act,\tok\trans,\tok\cons,\tok\reloads,\ca) \) such that:
\begin{itemize}
\item \( \tok\states \) contains all tuples \( (\bels,\guess) \) s.t. \( \bels\subseteq 2^\states \) is a belief support in \( \cpomdp \) and \( \guess \) is either an element of \( \bels \) or a special symbol \( \eguess \) (“empty guess");
\item for each \( (\bels,\guess) \in \tok\states \) and each \( a\in \act \) we have that \( \tok\trans((\bels,\guess),a) \) is a uniform distribution over the set \( \tok\mysucc((\bels,\alpha),a) \) defined as follows:
\begin{itemize}
\item if \( \alpha \neq \varepsilon \), we add to \( \tok\mysucc((\bels,\alpha),a) \) all tuples of the form \( (\bels',\alpha') \), where \( \bels' \in \mysucc(\bels,a)\) and \( \alpha' \) satisfies one of the following: either \( \alpha'\in\mysucc(\alpha,a) \cap \bels' \) or \( \mysucc(\alpha,a) \cap \bels' = \emptyset \) and \( \alpha' = \varepsilon \);
\item if \( \alpha = \varepsilon \), we add to \( \tok\mysucc((\bels,\alpha),a) \) all tuples of the form \( (\bels',\varepsilon) \), where \( \bels' \in \mysucc(\bels,a)\)
\end{itemize}
\item for each \( (\bels,\alpha)\in \tok\states \) and each \( a \in \act \) we put \( \tok\cons((\bels,\alpha),a) = \cons(s,a) \) where \( s \) is an arbitrary element of \( \bels \) (this definition is correct since \( \cpomdp \) is consistent).
\item \( \tok\reloads \) contains those tuples \( (\bels,\alpha)\in \tok\states \) such that \( \bels \subseteq \reloads \).
\end{itemize}

\paragraph{Pruning Token CoMDPs and Computing Exact Shields.}

It is \emph{not correct} to directly apply the algorithm of~\cite{BlaBNOTT:2020:cons-mdp-buchi} to the token MDP, as the following example shows.

\begin{example}
Consider the CoPOMDP \( \cpomdp \) pictured in the  left part of Figure~\ref{fig-gadgets}, with the corresponding   token CoMDP in the right (the position of the token is given by the hat symbol, e.g., \( \{\hat{p},q\} \) represents the state \( (\{p,q\},p) \)).
There exists no safe goal policy from \(s \) in \(\cpomdp \), since there is always a chance of getting stuck in \( t \): after one step, the agent cannot know whether it is in \( p \) or \( q \) and hence whether to choose action \( a \) or \( b \).  But in the token CoMDP such a policy clearly exists. 
\end{example}

\begin{figure}[t]\centering
	\begin{tikzpicture}[scale=.85, every node/.style={scale=0.85}, x=2.1cm, y=2.1cm, font=\footnotesize]
		\newcommand{\colortikz}{blue}
		\begin{scope}[shift={(0,0)}]
			\node[state] (s) at (0,0) {\(s\)};
			\draw [rounded corners,draw,dotted, very thick] (.25,-.75)  rectangle (.75,.85); 
			\draw [rounded corners,draw,dotted, very thick] (1,-.75)  rectangle (1.5,.85); 
			\node[state] (p) at (0.5,0.5) {\(p\)};
			\node[state] (q) at (0.5,-0.5) {\(q\)};
			\node[state,accepting] (r) at (1.25,0.5) {\(r\)};
			
			\node[state] (g) at (2,0.5) {\(g\)};
			\node[state,accepting] (t) at (1.25,-0.5) {\(t\)};
			\draw [tran] (s) to  (p); 
			\draw [tran] (s) to (q);
			\draw [tran,loop above] (t) to  (t);
			\draw [tran,loop above] (r) to  (r);
			\draw [tran] (p) to  (r);
			\draw [tran] (q) to  (t);
			\draw [tran] (r) to  (g);
			\draw [tran, dashed] (p) to  (t);
			\draw [tran, dashed] (q) to  (r);
			\draw [tran,loop above] (g) to  (g);
			
		\end{scope}
		\begin{scope}[shift={(2.3,0)}]
			\node[state] (s) at (0,0) {\(\{\widehat{s}\}\)};
			\node[state] (p) at (0.5,0.5) {\(\{\widehat{p},q\}\)};
			\node[state] (q) at (0.5,-0.5) {\(\{p,\widehat{q}\}\)};
			\node[state,accepting] (r) at (1.25,0.5) {\(\{\widehat{r},t\}\)};
			
			\node[state] (g) at (2,0.5) {\(\{\widehat{g} \}\)};
			\node[state,accepting] (t) at (1.25,-0.5) {\(\{r,\widehat{t} \}\)};
			\draw [tran] (s) to  (p); 
			\draw [tran] (s) to (q);
			\draw [tran,loop above] (t) to  (t);
			\draw [tran,loop above] (r) to  (r);
			\draw [tran] (p) to  (r);
			\draw [tran] (q) to  (t);
			\draw [tran] (r) to  (g);
			\draw [tran, dashed] (p) to  (t);
			\draw [tran, dashed] (q) to  (r);
			\draw [tran,loop above] (g) to  (g);
		\end{scope}
	\end{tikzpicture}
	\caption{A CoPOMDP (left) and it's corresponding token CoMDP (right). States \(r,t \) (double circles) are reloads, and \(g \) is a goal. All actions consume a single unit of the resource, apart from the self-loop on \( g \), which has zero consumption. There are two actions, \( a,b \). Solid edges represent transitions under action \(a \) (with uniform branching in \( r \)) and dashed edges represent action \(b \) (if both actions behave in the same way in a given state, only the \( a \)-edges are pictured). States enclosed in dotted rectangles are indistinguishable, i.e., always emit the same common observation. 
	}
	\label{fig-gadgets}
\end{figure}
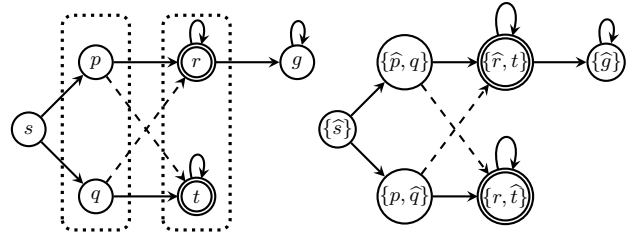

 Instead, we iteratively “prune" the token CoMDP \(\tok\cpomdp = (\tok\states,\act,\tok\trans,\tok\cons,\tok\reloads,\ca) \) by iteratively removing  the reloading property from all pairs \( (\bels,\alpha) \) that correspond to a trap, until reaching a fixed point. The  resulting CoMDP can be used to compute threshold values in the original CoPOMDP and thus also exact shield for \( \cpomdp \). The process is summarized in Algorithm~\ref{algo:shield}. Lines \ref{algline-rep-start}--\ref{algline-rep-end} compute the threshold levels, the remaining lines extract the shield. The latter part uses an ``inverse'' \( \resupinv \colon \nint{\ca} \times \states \times \act \rightarrow \nint{\ca} \cup \{\infty\} \) of the function \( \resup \) s.t. \( \resupinv(\ell', s, a) \) can be interpreted as the minimal amount of resource we need to have in \( s \) so that after playing \( a \) we have at least \( \ell' \) units. 
 Formally, 
 if \( \ell'=\bot \), also \( \resupinv(\ell', s, a)  = \bot\), irrespective of \( s, a \). Otherwise, 
 \[
 \resupinv(\ell', s, a) = \begin{dcases}
  \ell' + \cons(s, a) & \ell' \leq \ca - \cons(s,a) \wedge s \not \in \reloads,\\
 0 & \ell'\leq \ca - \cons(s,a) \wedge s \in \reloads,\\
 \infty & \ell' > \ca - \cons(s,a) .\\
 \end{dcases}
 \]

\begin{algorithm}
	\KwIn{consistent CoPOMDP \( \cpomdp \)}
	\KwOut{succinct support-based exact shield \( \shield \) for \( \cpomdp \)}
	compute the token CoMDP \( \tok\cpomdp  = (\tok\states,\act,\tok\trans,\tok\cons,\tok\reloads,\ca) \)\\
	\Repeat{\( \toremove = \emptyset \)\label{algline-rep-end}}{\label{algline-rep-start}
		\( \toremove \leftarrow \emptyset \);\\
		compute \( \tlevPR^{\tok\cpomdp} \)\label{aldline-compute-cmdp} \tcc*{Theorem~\ref{thm:cmdp-tvals-algo}}
		\ForEach{\( (\bels,\alpha) \in\tok\reloads\)}{ \If{\( \tlevPR^{\tok\cpomdp}(\bels,\alpha) = \infty \)}{
				\ForEach{\( \beta \in \bels \cup\{\varepsilon\} \)}{\( \toremove \leftarrow \toremove\cup\{(\bels,\beta)\} \)}
		}}
		\( \tok\cpomdp  \leftarrow  (\tok\states,\act,\tok\trans,\tok\cons,\tok\reloads\setminus\toremove,\ca)\)\label{algline-ct-rel-removed}
	} 
	\ForEach{\( \bels \subseteq 2^\states, a \in \act\) }{
	    \( \textit{MAX} \leftarrow -\infty \);\\
		\ForEach{\( s \in \bels \)}
		{\( \textit{SMAX} \leftarrow -\infty \);\\
		\ForEach{ \( \bels' \in \bsucc(\bels, a, s)\)}{
			\( \alpha'\leftarrow \text{ any elem. of \( \bels' \)}\);\\
			\( \textit{SMAX} \leftarrow \max (\textit{SMAX}, \tlevPR^{\tok\cpomdp}(\bels',\alpha') ) \);\\}
		\( \textit{MAX} \leftarrow \max (\textit{MAX}, \resupinv(\textit{SMAX}, s, a )) \)
		
	}
	\( \shield(\bels, a) \leftarrow \textit{MAX} \)
	}
	\Return{\( \shield \)}
	\caption{Computing exact shields.}
	\label{algo:shield}
\end{algorithm}

\begin{theorem}\label{thm:shield-al-correct}
After the repeat cycle in Algorithm~\ref{algo:shield} finishes, for any \( \bels \subseteq 2^\states \) and \( \alpha,\beta \in \bels \) it holds \( \tlevPR^{\tok\cpomdp}(\bels,\alpha) = \tlevPR^{\tok\cpomdp}(\bels,\beta) \). Moreover, the
computed \(\shield \) is the unique (support-based succinct) exact shield for \( \cpomdp \). There is a safe goal policy (i.e., the RSGO problem admits a feasible solution) iff \( \shield \) enables at least one action for the history of length~0.
The algorithm runs in time \( \mathcal{O}(2^\states\cdot\mathit{poly}(||\cpomdp||)) \), where \( ||\cpomdp|| \) denotes the encoding size of \( \cpomdp \) (with all integers encoded in binary).
\end{theorem}

\section{RSGO Problem and Shielded POMCP}
\label{sec:fipomdp}

We tackle the RSGO problem by augmenting the POMCP algorithm with our shields.

\paragraph{POMCP.} POMCP~\cite{SV:2010:POMCP} is a well-known online planning algorithm for POMDPs, based on the \emph{Monte Carlo tree search (MCTS)} paradigm. To select an optimal action, POMCP iteratively searches through and expands the \emph{history tree} of a POMDP, whose nodes correspond to histories. 
Each iteration consists of a top-down traversal of the explored part of the tree, selecting simulated actions according to the UCT formula~\cite{KS06} which balances exploitation with exploration. Once the simulation reaches a yet unexplored node, a \emph{rollout} policy (typically selecting actions at random) is used to further extend the sampled trajectory. The trajectory is then evaluated and the outcome is back-propagated to adjust action-value estimates along the sampled branch. After multiple such iterations, POMCP selects action with minimal value estimate in the root to be played by the agent. After receiving the next observation, the corresponding child of the root becomes the new root and the process repeats until a decision horizon is reached. 

\paragraph{Search Tree for Shielding.} We augment POMCP's tree data structure so that each node additionally contains the information about the current resource level and the current belief support (which is computable for each node using the belief support of the parent.) 

\paragraph{FiPOMDP.} Combining a (support-based exact)  shield \( \shield \) with POMCP yields an algorithm which we call FiPOMDP (``Fuel in POMDPs''). FiPOMDP operates just like POMCP, with one crucial difference: whenever POMCP is supposed to select an action in a node representing history \( h \), FiPOMDP chooses only among actions \( a \) such that \( \shield(\bels(h),a) = \en \). This applies to the simulation/tree update phase (where it selects action optimizing the UCT value among all allowed actions), rollouts, and final action selection (where it chooses the allowed action with minimal value estimate). Since POMCP is an online algorithm operating over a finite decision horizon, FiPOMDP solves the RSGO problem in the following approximate sense:


\begin{theorem}
\label{thm:fipomdp}
Let \( N \) be the decision horizon and consider a finite horizon approximation of the RSGO problem where the costs are accumulated only over the first \( N \) steps. 
Consider any decision step of FiPOMDP and let \( h \) be the history represented by the current root node of the search tree. Let \( p_h \) be the probability that the action selected by POMCP to be played by the agent is an action used in \( h \) by an optimal finite-horizon safe goal policy, and \( \mathit{sim} \) the number of simulations used by FiPOMDP. Then for \( \mathit{sim} \rightarrow \infty \) we have that \( p_h \rightarrow 1 \).
\end{theorem}

\paragraph{Shielding Other Algorithms.} The shields themselves are algorithm-agnostic and their usage is not limited to POMCP (or to MCTS algorithms in general). Indeed, one of the advantages of our approach is that shields can be used with any algorithm that tracks the current resource level and the current belief support.

\noindent

%
%
%
%







%




\section{Experiments}

\begin{table*}[t!]
\centering
\begin{tabular}{@{}p{1pt}lccccccc@{}}
\toprule
& & \# States & \# Obs & Survival \%& Hit \% & Avg. cost & Avg. time p. dec. (s)  & Shield time (s) \\
\midrule
\multicolumn{4}{l}{\textbf{FiPOMDP (shielded)}}\\
\midrule[0pt]    
&Tiger simple & \( 8 \) & \( 6 \) & \( 100 \) & \( 99.5 \)& \( 310.61 \pm 1942.14 \)& \( 0.05 \pm 0.02 \) & \( <1 \) \\
&Tiger fuzzy& \( 8 \) & \( 6 \) & \( 100 \) & \( 57.3 \) & \( 1019.95 \pm 2031.03 \) & \( 0.08 \pm 0.03  \) &\( <1 \) \\
&UUV grid 8x8 & \( 64  \) & \( 64 \) & \( 100 \) & \( 98 \) & \( -969.81 \pm 153.61 \) & \( 3.21 \pm 1.44 \)& \( 10.65 \)\\
&UUV grid 12x12 & \( 144 \) & \( 144 \)& \( 100 \) &  \( 92 \) &  \( -898.48 \pm 295.92 \) & \( 10.48\pm 3.59 \)  & \( 76.99 \)\\
&UUV grid 16x16 & \( 256 \) & \( 256 \) & \( 100 \) & \( 87 \) & \( -839.29 \pm 364.92 \)& \( 22.21 \pm 5.52 \) & \( 215.18 \)\\
&UUV grid 20x20 & \( 400 \) & \( 400 \)& \( 100 \) & \( 87 \) & \( -839.08 \pm 364.83 \) & \( 34.06 \pm 8.87\) & \( 493.87 \)\\
&Manhattan AEV & \( 22434 \) & \( 7478 \) & \( 100 \) & \( 50 \)& \( 2745.3 \pm 2845.06 \) & \( 13.02 \pm 2.44 \) & \( 65 \)\\
\midrule[0pt]
\multicolumn{4}{l}{\textbf{POMCP (unshielded)}}\\
\midrule[0pt]  
&Tiger simple & \( 8 \) & \( 6 \) & \( 99.9 \) & \( 99.4 \)& \( 311.59 \pm 1941.99 \)& \( 0.02 \pm 0.01 \) & - \\
&Tiger fuzzy& \( 8 \) & \( 6 \) & \( 86.8 \) & \( 48.5 \) & \( 914.87 \pm 1859.81 \) & \( 0.02 \pm 0.01  \) &- \\
&UUV grid 8x8 & \( 64  \) & \( 64 \) & \( 61 \) & \( 60 \) & \( -555.75 \pm 538.11 \) & \( 1.91 \pm 0.27 \)& -\\
&UUV grid 12x12 & \( 144 \) & \( 144 \)& \( 8 \) &  \( 7 \) &  \( 23.76 \pm 279.29 \) & \( 9.49\pm 1.23 \)  & -\\
&UUV grid 16x16 & \( 256 \) & \( 256 \) & \( 4 \) & \( 3 \) & \( 67.52 \pm 185.62 \)& \( 18.23 \pm 1.60 \) & -\\
&UUV grid 20x20 & \( 400 \) & \( 400 \)& \( 6 \) & \( 5 \) & \( 45.77 \pm 237.57 \) & \( 28.94 \pm 2.96\) & -\\
&Manhattan AEV & \( 22434 \) & \( 7478 \) & \( 99 \) & \( 56 \)& \( 2352.88 \pm 2935.04 \) & \( 13.36 \pm 2.38 \) & -\\
\bottomrule
\end{tabular}
\caption{Results of experiments. The top part shows results for FiPOMDP, the bottom for the POMCP baseline.}
\label{tab:results}
\end{table*}

We implemented FiPOMDP in Python. The algorithm for exact shield computation was implemented on top of the planning algorithm for CoMDPs.~\cite{BlaBNOTT:2020:cons-mdp-buchi}. We wrote our own implementation of POMCP, including the particle filter used for belief approximation~\cite{SV:2010:POMCP}. Our implementation can be found at~\url{https://github.com/xbrlej/FiPOMDP}.

\paragraph{Benchmarks.}

We evaluated FiPOMDP on three sets of benchmarks. 
The first benchmark is a toy \emph{resource-constrained Tiger,} a modification of the classical benchmark for POMDPs~\cite{kaelbling1998planning} adapted from~\cite{DBLP:journals/corr/BrazdilCCGN16}. The goal states represent the situation where the agent has made a guess about the tiger's whereabouts. In the resource-constrained variant, the agent's listening actions consume energy, necessitating regular reloads. During each reload, there is a probability that the tiger switches its position. There is a cost of 10 per each step, opening door with tiger/treasure yields cost 5000/-500. We consider two versions: \emph{simple}, where the probability of the observed position of the tiger being correct is \( 0.85 \), and \emph{fuzzy}, where this probability is decreased to \( 0.6 \).

The second benchmark is a partially observable extension of the \emph{unmanned underwater vehicle (UUV)} benchmark from~\cite{BlaCNOTT:2021:FiMDP}. Here, the agent operates in a grid-world, with actions performing movements in the cardinal directions. 
Movement is subject to stochastic perturbations: the UUV might drift sideways from the chosen direction due to ocean currents. The position sensor is noisy: when the agent visits some cell of the grid, the observed position is sampled randomly from cells in the von Neumann neighbourhood of the true cell. We consider 4 gridworld sizes ranging from 8x8 to 20x20. There is a cost of 1 per step, hitting a goal yields ``cost'' -1000.

The final benchmark, adapted from~\cite{BlaBNOTT:2020:cons-mdp-buchi}, consists of a routing problem for an autonomous electric vehicle \emph{(AEV)} in the middle of Manhattan, from 42nd to 116th Street. Road intersections act as states. At each intersection, the AEV picks a direction to continue (subject to real-world one-way restrictions). It will deterministically move in the selected direction, but the energy consumption is stochastic due to the fluctuations in the road congestion. There are three possible consumption levels per road segment, their probability and magnitudes derived from real-world traffic data.~\cite{Uber:2019,Tesla:2008}. Similarly, the reload states correspond to the real-world positions of charging stations~\cite{evdatabase}. To add partial observability, we make the consumption probabilities dependent on the unknown \emph{traffic state} (low/medium/peak) which evolves according to a known three-state Markov chain. The cost is equal to the amount of resource consumed in a given step, with a ``cost'' -1000 when a goal is hit.

\paragraph{Evaluation.} The hardware configuration was: CPU: AMD Ryzen 9 3900X (12 cores); RAM: 32GB; Ubuntu 20.04. 

FiPOMDP is the first approach to solving the RSGO problem in CoPOMDPs. Hence, as a baseline to compare with we chose plain (unshielded) POMCP (with the same hyperparameters),  to see how the formal safety guarantees of FiPOMDP influence resource safety in practice. POMCP itself does not consider resource levels, which puts it at a disadvantage. To mitigate this, we treated (only in the POMCP experiments) resource exhaustion as entering a “breakdown" sink state, from which the target can never be reached. Hence, runs exhausting the resource were penalized with the same cost as runs which did not reach the goal.\footnote{The final archival version will also provide a comparison with a version where resource exhaustion receives large cost penalty.}

The results are pictured in Table~\ref{tab:results}, averaged over 100 runs (1000 for the Tiger benchmark). The first two columns show the number of states and observations. The \emph{Survival \%} is the percentage of runs in which the agent \emph{did not} run out of the resource. The \emph{Hit \%} is the percentage of runs in which the agent hit the target within the decision horizon. The next column shows an average cost incurred by the agent (\( \pm \) the std. deviation). We also present average time per decision step. The final column shows the time needed to compute the shield (including  computation of the token CoMDP).

We highlight the following takeaway messages: (1.) Although computing an exact shield requires formal methods, our algorithm computed a shield within a reasonable time even for relatively large (from a formal methods point of view) CoPOMDPs (the Manhattan benchmark). 
(2.) Shielding is essential for resource safety. The unshielded version never achieved 100\% resource safety. In constrast, FiPOMDP \emph{never} exhausted the resource, validating its theoretical guarantees. (3.) The \emph{Hit percentage} and \emph{Cost} results show that the shielded POMCP is consistently able to reach the goal. 
On the other hand, the hit ratios are sometimes not as high as desired. We suspect that this is because our benchmarks are ``non-smooth'' in the sense that the costs encountered \emph{before} reaching the goal do not provide much information about a \emph{path} towards the goal. This was partially mitigated using \emph{heavy rollouts} (in particular for the gridworld benchmark, where we used non-uniform rollouts with an increased likelihood of the agent repeating the direction chosen in the previous step). (4.) Since unshielded POMCP tends to exhaust the resource, FiPOMDP has (in all but one of the benchmarks) clearly better hit percentage than the unshielded POMCP. In relatively structureless domains, such as the gridworld, shielding seems to help exploring the state space by pruning away parts from which resource exhaustion cannot be prevented. (5.) The Manhattan benchmark stands out in that here, the unshielded version performs better in terms of “Hit \%” than the shielded one. Still, the unshielded version still is not 100\& safe. The benchmark admits a policy for quickly reaching the goal which carries a small risk of resource exhaustion. The unshielded agent takes this policy, while the shielded agent computes a policy which is resource-safe at the cost of slower progress towards the goal. This shows that shields protect even against relatively low (though practically significant) exhaustion risks.

\section{Conclusion}

We presented a shielding algorithm for consumption POMDPs with resource safety and goal reachability objectives. We combined our shields with the POMCP planning algorithm, yielding a heuristic approach to solving the RSGO problem. An interesting direction for the future work is to combine our shielding algorithm with alternative approaches to POMDP planning.

\section*{Acknowledgments}

This work is supported by the Czech Science Foundation grant No. 21-24711S. We thank all the anonymous reviewers for providing feedback on the preliminary versions of this paper.

\bibliography{biblio}


\clearpage
\appendix
\begin{center}
{\Large Technical Appendix}
\end{center}

\newtheorem{invariant}{Invariant}

\newtheorem*{lemma*}{Lemma}

\newtheorem*{theorem*}{Theorem}

\section{Belief Supports for CoPOMDPs}

In the following, we extend belief supports to handle information about resource levels.

 The formal definition of \( \belsupf(h) \) is by induction on the length of \( h \): if \( \len(h)=0 \), then \( h = \ob_0\ell_0 \) for some \( \ob_0 \in \obs \) and \( \ell_0\in \nint{\ca} \), and we put 
\[
\belsupf(h) = \supp(\initdistr) \cap \{s \mid \ob_0 \in \supp(\obsfunc(s))\}.
\]
where \(\initdistr\) in the initial distribution on states. If \( \len(h)>0 \), then \( h = \hat{h}a\ob\ell  \) for some shorter history \( \hat{h} \).
 To define \( \belsupf(h) \), let \( \hat\ob\in \obs \) be the last observation of \( \hat{h} \) and \( \hat{\ell}\) the last resource level of \( \hat{h} \). For a state \( s \in \belsupf(\hat{h}) \), let \( \hat{\ell}_s  \) be \( \ca \) if \( s \in \reloads \) or \( \hat{\ell} \) if \( s \not \in \reloads \). We say that \( s \) \emph{conforms} to \( h \) if one of the following conditions holds:
 \begin{itemize}
 \item \( \hat{\ell} = \bot \); or
 \item \( \hat{\ell} \neq \bot, \ell = \bot  \), and \( \cons(s,a) \geq \hat{\ell}_s \); or
  \item \( \hat{\ell} \neq \bot \neq \ell \), and \( \ell = \hat{\ell}_s - \cons(s,a) \).

 \end{itemize} 
 Now, given \( \belsupf(\hat{h}) \), for \( h = \hat{h}a\ob\ell \) we define \( \belsupf(h) \) to be the set \( \bigcup_{{s\in \belsupf(\hat{h}), s \text{ conforms to \( h \)}}}\mysucc(s,a)\cap\{t\mid \ob \in \supp(\obsfunc(t))\} \).

 \section{Preprocessing for Consistent Consumption}

We can turn a non-consistent CPOMDP \(\cpomdp \) into a consistent CPOMDP \(\widehat{\cpomdp} \)  by "inserting" a new state into the "processing" of each action as shown in Fig.~\ref{fig-preprocessing-consistent-cons}. Let \(s\in \states \) and \(a\in \act \), then we insert a new state \(t_{s,a} \) such that from \(s \) under \(a \) we first reach \(t_{s,a} \), and then by using any action in \(t_{s,a}  \), we "finish resolving" the original action \(a \) on \(s \) thus arriving at the resulting state (chosen according to \(\trans(s,a) \)) and "consuming" the resource. If we define the observation of \(t_{s,a} \) as the value \(\cons(s,a) \), we effectively obtain additional observation equal to the amount of consumed resource whenever an action is being resolved (note that this does not give us any new information as we can already always calculate the exact consumption). The resulting \(\widehat{\cpomdp} \) is then consistent, as resource can only be decreased by the value equal to the current observation (or \(0\) in the original states of \(\cpomdp \)).

Formally we construct \(\widehat{\cpomdp}=(\widehat{\states},\act,\widehat{\trans},\widehat{\obs},\widehat{\obsmap},\widehat{\cons},\reloads,\ca) \) such that:
   
\begin{itemize}
	\item \( \widehat{\states}= \states \cup T\) where \(T=\{t_{s,a}  \mid s\in \states, a\in \act  \} \);
	\item for each \( s\in \states \), \(a\in \act \), \(t_{s',b}\in T \) and \(p\in \widehat{\states} \) we define
	\[\widehat{\trans}(s,a)(p)= \begin{cases}1   & \text{if } p=t_{s,a} \\
		0                                    & \text{else }
	\end{cases}\]
\[\widehat{\trans}(t_{s',b},a)= \trans(s',b); \]

\item \(\widehat{\obs} =\obs \cup \nint{\ca}  \), where we assume that \(\obs \cap \nint{\ca} = \emptyset \);
\item for each \(s\in \states \), \(t_{s',a}\in T\) and \(o\in \widehat{\obs} \) we define \[\widehat{\obsmap}(s)=\obsmap(s) \]\[\widehat{\obsmap}(t_{s',a})(o)=\begin{cases}1   & \text{if } o=\cons(s',a) \\
	0                                    & \text{else }
\end{cases}  \] 

	\item for each \(s\in \states \), \(a\in \act \) and \(t_{s',b}\in T \) we define \(\widehat{\cons}(s,a)=0 \) and \(\widehat{\cons}(t_{s',b},a)=\cons(s',b) \).
\end{itemize}

\section{Proof of Lemma~\ref{lem:thr-level-shielding}}

\begin{figure}[t]\centering
	\begin{tikzpicture}[scale=.85, every node/.style={scale=0.85}, x=2.1cm, y=2.1cm, font=\footnotesize]
		\begin{scope}[shift={(0,0)}]
			\node[state] (s) at (0,0) {\(s\)}; 
			\node[state] (p1) at (0.65,0.5) {\(p_1\)};
			\node[state] (pk) at (0.65,-0.5) {\(p_k\)};
			
			\draw [tran] (s) to  (p1); 
			\draw [tran] (s) to (pk);
			\pic [draw, -,"$a$", angle eccentricity=1.5] {angle = pk--s--p1};
			\draw[dotted,shorten <=0.25cm,shorten >=0.25cm] (p1) to (pk);
			
			\draw [tran] (1.1,0) to (1.85,0);
		\end{scope}
		\begin{scope}[shift={(2.5,0)}]
			
			\node[state] (s) at (0,0) {\(s\)};
			\node[state] (t) at (0.65,0) {\(t_{s,a}\)}; 
			\node[state] (p1) at (1.3,0.5) {\(p_1\)};
			\node[state] (pk) at (1.3,-0.5) {\(p_k\)};
			
			\draw [tran] (s) to node[above] {a} (t); 
			\draw [tran] (t) to  (p1); 
			\draw [tran] (t) to (pk);
			\pic [draw, -,"$act$", angle eccentricity=1.75] {angle = pk--t--p1};
			\draw[dotted,shorten <=0.25cm,shorten >=0.25cm] (p1) to (pk);
		
		\node[state] (s) at (0,0) {\(s\)};
		\node[state] (t) at (0.65,0) {\(t_{s,a}\)}; 
		\node[state] (p1) at (1.3,0.5) {\(p_1\)};
		\node[state] (pk) at (1.3,-0.5) {\(p_k\)};
		
		\draw [tran] (s) to node[above] {a} (t); 
		\draw [tran] (t) to  (p1); 
		\draw [tran] (t) to (pk);
		\pic [draw, -,"$act$", angle eccentricity=1.75] {angle = pk--t--p1};
		\draw[dotted,shorten <=0.25cm,shorten >=0.25cm] (p1) to (pk);
		\end{scope}
	\end{tikzpicture}
	\caption{"inserting" new state \(t_{s,a} \) into the action \(a \) from state \(s \), where \(act \) represents arbitrary action.}
	\label{fig-preprocessing-consistent-cons}
\end{figure}
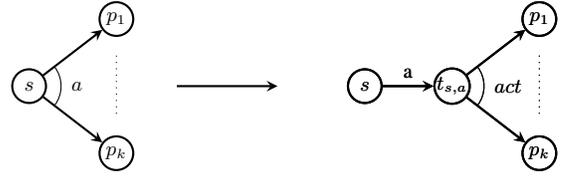


\begin{lemma*}[\textbf{1}]
	Let \( \shield \) be an exact shield. Then for each \( h\in \hist \) and each \( a \in \act \) we have that \( \shield(h,a) = \en \) if and only if \( \ell_h \) is greater than or equal to the smallest number \( \tau \) s.t. for any \( s \in \belsupf(h) \) and any valid history of the form \( ha\ob\ell \) s.t. \( \obsfunc(\ob | t)>0  \) for some \( t \in \mysucc(s, a) \) it holds that \( \resup(\tau, s, a) \geq \tlev(ha\ob\ell) \).
\end{lemma*}

\begin{proof}
	"\(\Rightarrow \)" Assume that there exists an exact shield \(\shield \), \(h\in \hist \) and \(a\in \act \) such that \(\shield(h,a)=\en \) and \(\ell_h< \tau \).
	
	Then there exist \( s \in \belsupf(h), t\in \mysucc(s,a)\), and \(\ob \in \supp(\obsfunc(t)) \) such that \( \resup(\ell_h, s, a) < \tlev(ha\ob\ell) \). 
	
	Since \(\shield(h,a)=\en \) and \(\shield \) is exact, it holds that \((\belsupf(ha\ob\ell),\ell) \) is not a trap. Therefore there exists a safe goal policy for history \(ha\ob\ell \), which is a contradiction with \(\resup(\ell_h, s, a)< \tlev(ha\ob\ell)\).


	"\(\Leftarrow \)" Assume that there exists an exact shield \(\shield \), \(h\in \hist \) and \(a\in \act \) such that \(\shield(h,a)=\dis \) and \(\ell_h\geq 
	\tau \).
	
	
	Then let \(s_1t_1\ob_1, \dots, s_kt_k\ob_k  \) be all the possible values of \(st\ob \), and let \(\ell_i=\resup(\reslevrv_{h},s_i,a) \). As it holds  \(\ell_i\geq \tlev(ha\ob_i\ell_i)  \) for all \(i \), there exist safe goal policies \(\pi_1,\dots,\pi_k \) for histories \( ha\ob_1\ell_1,\dots,ha\ob_k\ell_k \) respectively. Now let us consider the policy \(\pi \) which for history \( h\) chooses action \( a\), and if the next history is \(ha\ob_j\ell_j \) it then proceeds as the policy \(\pi_j \). We claim that \(\pi\) is a safe goal policy from history \(h\). As any history produced using \(\pi \) from \(h\) is also a history produced by \(\pi_j \) from \(ha\ob_j\ell_j\) for some \(j\), if \(\pi \) were not to be safe then either \(\pi \) can run out of resource during the first step, or there exists \(\pi_j \) that is not safe from it's corresponding history. The former being a contradiction with \(\ell_j\geq 0 \) for all \(j \), while the latter is a contradiction with each \(\pi_j \) being safe from the corresponding history. Therefore \(\pi\) is safe. And since each of \(\pi_j \) is a goal policy, \(\pi \) is also a goal policy. Therefore there exists a safe goal reaching policy for history \(h\) which chooses action \(a\), which is a contradiction with \(\shield\) being exact.
\end{proof}

\section{Proof of Theorem~\ref{thm:shield-al-correct}}

\begin{theorem*}[\textbf{2}]
	After the repeat cycle in Algorithm~\ref{algo:shield} finishes, for any \( \bels \subseteq 2^\states \) and \( \alpha,\beta \in \bels \) it holds \( \tlevPR^{\tok\cpomdp}(\bels,\alpha) = \tlevPR^{\tok\cpomdp}(\bels,\beta) \). Moreover, the
	computed \(\shield \) is the unique (support-based succinct) exact shield for \( \cpomdp \).  There is a safe goal policy (i.e., the RSGO problem admits a feasible solution) iff \( \shield \) enables at least one action for the history of length~0.
	The algorithm runs in time \( \mathcal{O}(2^\states\cdot\mathit{poly}(||\cpomdp||)) \), where \( ||\cpomdp|| \) denotes the encoding size of \( \cpomdp \) (with all integers encoded in binary).
	
\end{theorem*}

In the following we use \(\tok\cpomdp'=(\tok\states,\act,\tok\trans,\tok\cons,\tok\reloads',\ca) \) to denote the "pruned" \(\tok\cpomdp \) from algorithm~\ref{algo:shield} (i.e. \(\tok\cpomdp' \) is equal to \(\tok\cpomdp \) after the last iteration of the repeat cycle at Lines~\ref{algline-rep-start}-\ref{algline-rep-end}). 
We denote by \( \tlevPR^{\tok\cpomdp'}(\bels) \) the common \( \tlevPR^{\tok\cpomdp'} \)-value of all tuples of the form \( (\bels,\alpha), \alpha\neq \varepsilon \) (the existence of such value is proven in Lemma~\ref{app:lemma-safe-produce-positive}).

%
%



\begin{lemma}\label{app:lemma-tlev-is-tlevpr}
	For any history \(h\in \hist \) it holds \( \tlev^\cpomdp(h) = \tlevPR^{\tok\cpomdp'}(\belsupf(h)) \).
\end{lemma}
%
%

This follows directly from Lemmas~\ref{app:lemma-tlev-leq-tlevpr} and~\ref{app:lemma-tlev-geq-tlevpr} that are stated and proven below.

Therefore from Lemma~\ref{lem:thr-level-shielding} we get that the shield \(\shield \) computed by algorithm~\ref{algo:shield} is the unique support-based exact shield for \( \cpomdp\).

	As each belief support is a subset of \(\states \) it holds that \(|\tok\states|\) is in  \(\mathcal{O}(2^{|\states|}) \). The number of iterations of the repeat cycle at lines~\ref{algline-rep-start}-\ref{algline-rep-end} is bounded by \(|\tok\reloads|\leq |\tok\states| \), as each iteration is either the last one or at least one reload is removed from \(\tok\reloads\). 
	Removing reloads and computing the token MDP can both be done in time \(\mathcal{O}(2^{|\states|}) \).
	Computing \(\tlevPR^{\tok\cpomdp} \) at line~\ref{aldline-compute-cmdp}  takes time which is polynomial in \(||\tok\cpomdp || \) as per Theorem~\ref{thm:cmdp-tvals-algo}. Therefore the running time of the algorithm~\ref{algo:shield} is  in \(\mathcal{O}(2^{|\states|}\cdot \mathit{poly}(||\tok\cpomdp||) ) \). Which can be expressed as  \(\mathcal{O}(2^{|\states|}\cdot2^{|\states|}\cdot \mathit{poly}(||\cpomdp||) ) \), as the two differ only in the sets \(\tok\states,\tok\trans,\tok\obs,\tok\obsmap,\tok\cons\) and \(\tok\reloads \) all of which have at most \(\mathcal{O}(2^{|\states|}) \) elements. 

Before proving Lemma~\ref{app:lemma-tlev-is-tlevpr}, we prove that \( \tlevPR^{\tok\cpomdp'}(\bels) \) is indeed well-defined.


\begin{lemma}\label{app:lemma-safe-produce-positive}
	For any \( \bels \subseteq 2^\states \) and \( \alpha,\beta \in \bels \) it holds \( \tlevPR^{\tok\cpomdp'}(\bels,\alpha) = \tlevPR^{\tok\cpomdp'}(\bels,\beta)\).  
	
	Furthermore if there is a safe policy \(\pi \) from history \((\bels,\alpha)\ell \) for any \(\alpha\in \bels\cup \{\varepsilon \} \) in \(\tok\cpomdp' \), then \(\tlevPR^{\tok\cpomdp'}(\bels,\beta)\leq \ell\) for all \(\beta\in \bels \), and \( \pi \) is a safe policy for  \((\bels,\varepsilon)\ell \).
\end{lemma}
\begin{proof}
	Let \(\bels \subseteq 2^\states\) be such that \(\alpha,\beta\in \bels \), and w.l.o.g. \( \tlevPR^{\tok\cpomdp'}(\bels,\alpha) < \tlevPR^{\tok\cpomdp'}(\bels,\beta)\). 
	
	
	Since every safe positive-goal policy is a safe policy it holds that for the initial history \((\bels,\alpha)\tlevPR^{\tok\cpomdp'}(\bels,\alpha) \) there exists a safe policy \(\sigma \). As \(\tok\cpomdp'\) is consistent, the consumption of every tuple \(\tok\cons((\bels,\alpha),a) \) is independent from \(\alpha \). And when considering reload states the token is irrelevant as well. Thus if we were to consider only safe policies, the position of the token is irrelevant. Therefore \(\sigma \) is also a safe policy for initial history \((\bels,\beta)\tlevPR^{\tok\cpomdp'}(\bels,\alpha) \). But we can use \(\sigma \) to construct a safe positive-goal policy \(\pi' \) for history \((\bels,\beta)\tlevPR^{\tok\cpomdp'}(\bels,\alpha) \)  as follows:
	Simulate \(\sigma \) until a reload \((\bels',\alpha')\in \tok\reloads' \) where \(\alpha'\neq \varepsilon \) is reached, and then simulate any safe positive-goal policy for initial history \((\bels',\alpha')cap \) (note that such policy always exists due to the way \(\tok\cpomdp'\) was constructed).
	
	Clearly \(\pi' \) is a safe policy, and so either the computation reaches a goal and never leaves, or it must eventually reach a reload in at most a bounded number of steps (otherwise it would run out of resource). If the probability of reaching a goal first is positive, then \(\pi' \) is a positive-goal policy. Let us therefore assume \(\pi' \) reaches a reload first with probability 1. If the reached reload were to be a state \((\bels',\alpha')\) for \(\alpha'\neq \varepsilon \) then this would mean that \(\pi' \) is positive-goal as this would mean \( \pi\) eventually behaves like \(\pi' \) with positive probability. Therefore for \(\pi' \) to not be positive-goal it has to hold  with probability 1 that the first reload \(\pi' \) reaches is of the form \((\bels',\varepsilon) \). However this is not possible due to Lemma~\ref{lem:guess-has-nonempty-succ}.
\end{proof}

\begin{lemma}\label{lem:guess-has-nonempty-succ}
	For any \(\bels \subseteq 2^\states, \alpha\in \bels\) and each \(a\in \act \) it holds \(\emptyset \neq  \{(\bels',\beta)\in \tok\mysucc((\bels,\alpha),a)\mid \beta\neq \varepsilon\} \).
\end{lemma}
\begin{proof} 
	Let \(s\in \mysucc(\alpha,a) \) and \(\ob\in \supp(\obsfunc(s)) \). Then there exists \(\bels'\in \belsupup(\bels,a,\ob,\cons(\alpha,a))\) such that \(s\in \bels' \), and therefore \(\bels'\in \mysucc(\bels,a) \). Which implies \((\bels',s)\in \tok\mysucc((\bels,\alpha),a) \).
\end{proof}

We are now ready to prove the main Lemma~\ref{app:lemma-tlev-is-tlevpr}. This is done via the two following lemmas.

\begin{lemma}\label{app:lemma-tlev-geq-tlevpr}
	For any history \(h\in \hist \) it holds \( \tlev^\cpomdp(h) \geq \tlevPR^{\tok\cpomdp'}(\belsupf(h)) \).
\end{lemma}
\begin{proof}
	Let \(\tok\reloads^0,\dots ,\tok\reloads^k \) be such that \(\tok\reloads^0=\tok\reloads \), \(\tok\reloads^k=\tok\reloads' \), and \(\tok\reloads^i \) is obtained by applying a single iteration of the repeat cycle at lines~\ref{algline-rep-start}-\ref{algline-rep-end} of Algorithm~\ref{algo:shield} onto \(\tok\reloads^{i-1} \). And let us denote by \(\tlevPR^i((\bels,\alpha)) \) the value \(\tlevPR^{\tok\cpomdp^i}((\bels,\alpha)) \) computed in \(\tok\cpomdp^i=(\tok\states,\act,\tok\trans,\tok\cons,\tok\reloads^i,\ca) \). 

	Assume we have a safe goal strategy \(\pi \) for \(h \) in \(\cpomdp\). We will show that \(\pi \) cannot generate a history \(\widehat{h}=hh' \) such that \((\belsupf(\widehat{h}),\alpha)\in \tok\reloads \setminus \tok\reloads' \) for any \( \alpha\). 

	Let us assume this is not the case, and let \(i>0 \) be the smallest such that \((\belsupf(\widehat{h}),\alpha)\in \tok\reloads \setminus \tok\reloads^i \) for some \(\alpha \). Then there exists \(\beta\in \belsupf(\widehat{h}) \) such that \(\tlevPR^{i-1}((\belsupf(\widehat{h}),\beta))=\infty \), and at the same time there is positive probability that the current state of \(\cpomdp \) after history \(\widehat{h} \) is \(\beta \).
	
	If we define an observation function over \(\tok\states \), such that the observation of  \((\bels,\alpha)\) is chosen uniformly at random from the set \(\{o\in \obs \mid \bigcap_{s\in \bels} \supp(\obsfunc(s)) \}\) (note that this set is never empty due to the nature of belief supports), then we can consider \(\tok\cpomdp^{i-1} \) as a CPOMDP with the same set of observations and actions as \(\cpomdp \), and therefore we can apply \(\pi \) to \(\tok\cpomdp^{i-1} \). Let us consider what happens if we run \(\pi \) on \(\tok\cpomdp^{i-1}\) from initial state \((\belsupf(\widehat{h}),\beta) \).  Since \(\tlevPR^{i-1}((\belsupf(\widehat{h}),\beta))=\infty \) this means such computation either has positive probability of running out of resource, or probability of never reaching the goal is 1. 
	
	As \(\pi \) is a safe goal policy in \(\cpomdp \), and \(\cpomdp \) is consistent,the resource level depends only on observations, which are the same in both \(\cpomdp \) as well as \(\cpomdp^{i-1} \), therefore the only way \(\pi \) runs out of resource in \(\cpomdp^{i-1} \) is if it reaches some reload \((\bels',\beta')\in \tok\reloads \setminus \tok\reloads^{i-1} \), but this is a contradiction with how we picked \(i \). Therefore it must hold that \(\pi \) never reaches a goal in \(\cpomdp^{i-1} \). But as \(\pi \) is a goal policy in \(\cpomdp \) there exists a finite path in \(\cpomdp \) under \(\pi \) from \(\beta \) to a goal   such that \(\pi \) has positive probability of producing this path in \(\cpomdp \) if started from history \(\widehat{h} \). Therefore in each step of the computation on \(\cpomdp^{i-1} \) there is a positive probability of the next state being \((\belsupf(h'),s') \), where \(h',s' \) are the next history/state pair of this path in \(\cpomdp \), and therefore  \(\tlevPR^{i-1}((\belsupf(\widehat{h}),\beta))<\infty \), a contradiction.
	
	Therefore \(\pi \) never produces a history \( h\) such that \((\belsupf(h),\varepsilon)\in \tok\reloads \setminus \tok\reloads' \). Therefore if we were to apply policy \(\pi \) onto \(\tok\cpomdp' \) from history \((\belsupf(h),\epsilon)\tlev(h) \), \(\pi\) would be a safe policy, and thus from Lemma~\ref{app:lemma-safe-produce-positive} we obtain \( \tlev^\cpomdp(h) \geq \tlevPR^{\tok\cpomdp'}(\belsupf(h)) \).
\end{proof}

\begin{lemma} \label{app:lemma-tlev-leq-tlevpr}
For any history \( h\in\hist \) it holds \( \tlev^\cpomdp(h) \leq \tlevPR^{\tok\cpomdp'}(\belsupf(h)) \).
\end{lemma}

\begin{proof}
Let \(\pi' \) be a policy that is safe positive-goal in \(\tok\cpomdp'\) for every history \((\bels,\alpha)\tlevPR^{\tok\cpomdp'}(\bels) \) where \(\tlevPR^{\tok\cpomdp'}((\bels))\neq \infty\) and \(\alpha\neq \varepsilon \), and let \(\sigma' \) be a policy that is safe for every history \((\bels,\alpha)\tlevPR^{\tok\cpomdp'}(\bels) \) in \(\tok\cpomdp'\) where \(\tlevPR^{\tok\cpomdp'}((\bels))\neq \infty\).
	
	Let us consider a policy \(\pi\) on the original CPOMDP which behaves as follows:
	\(\pi \) internally simulates a computation on \(\tok\cpomdp'\) and holds a counter of size \(|\tok\states|\cdot \ca \). For initial history \(h=h' \ob_0\ell_0 \) the simulation gets initialized to a state \((\belsupf(h),\epsilon) \) and the counter is set to \(0\). Also let us denote by \(\tok\mysucc((\bels,\alpha),a,\bels')\) the set \(\{(\widehat\bels,\beta)\in \tok\mysucc((\bels,\alpha),a) \mid \widehat\bels=\bels'  \} \). Let \((\bels,\alpha)\) and \(\ell \) be the current state and resource level of the simulation, respectively. Then \(\pi \) behaves as follows:
	\begin{itemize}
		\item If counter is \(0 \), then:
		\begin{itemize}
			\item if \((\bels,\alpha)\in \tok\reloads'\), then randomly at uniform choose \(\beta\in \bels \), set the current state of simulation to \((\bels,\beta)\) and set the counter to \(|\tok\states|\cdot \ca \).
			\item  if \((\bels,\alpha)\notin \tok\reloads'\) let \(a \) be the next action chosen by \(\sigma'\) in the simulation. Then output \(a\) as the next action in the original CPOMDP, and let \(h=h'a\ob\ell \) be the resulting history. Then change the state of the simulation to a state chosen uniformly at random from \(\tok\mysucc((\bels,\alpha),a,\belsupf(h))\) and set resource level to \(\resup(\ell,(\bels,\alpha),a)\).
		\end{itemize}
	\item if counter is not \(0 \) then let \(a \) be the next action chosen by \(\pi'\) in the simulation. Then output \(a\) as the next action in the original CPOMDP, and let \(h=h'a\ob\ell \) be the resulting history. Then change the state of the simulation to a state chosen uniformly at random from \(\tok\mysucc((\bels,\alpha),a,\belsupf(h))\) and set resource level to \(\resup(\ell,(\bels,\alpha),a)\). Finally decrease the counter by \(1 \).
	\end{itemize}

	The way \(\pi \) behaves is that whenever the counter is 0, it simulates the safe policy \(\sigma' \) until a reload is reached, at which point it tries to guess which state \(\cpomdp\) is in by placing a token at one of the states in the current belief support. Afterwards it simulates the safe positive-goal policy \(\pi' \) for \(|\tok\states|\cdot \ca \) steps (note that this is enough to give us a positive probability of reaching a goal) after which the counter is set to \(0 \) again and the whole process repeats. Since both \(\pi' \) and \(\sigma' \) are safe policies on \(\tok\cpomdp' \), and \(\cpomdp \) is consistent, \(\pi \) is also safe. And since each simulation of \(\sigma' \) can last at most a bounded number of steps (unless a goal is reached), the entire process can be repeated infinitely often until a goal is reached. Therefore \(\pi \) is a safe goal policy.
	More formal proof follows:
	
	We claim that \(\pi \) is a safe goal policy in \(\cpomdp \) for any history \(h=h'\ob_0\ell_0 \) where \(\ell_0\geq \tlevPR^{\tok\cpomdp'}(\belsupf(h))\).
	
	First we need to prove some invariants. Let \(h_i=\widehat{h}\ob_0\ell_0a_0\dots a_{i-1}\ob_i\ell_i \) be a history produced by \(\pi \) from \(h\) at time \( i\), and let  \((\bels_i,\alpha_i) \), \(\ell_i' \) be the state and resource level in the simulation at time \(i \), respectively. 
	
	\begin{invariant}\label{app:inv1}
		\( \belsupf(h_i)=\bels_i\)
	\end{invariant}
	\begin{proof}
		At time \(0 \) this holds. Assume it holds at step \(i-1 \). States in simulation can change only in two ways, either the token is moved within the same belief support, or the state changes from \((\bels,\alpha)\) to a state from \(\tok\mysucc((\belsupf(h_{i-1}),\alpha),a_{i-1},\belsupf(h_{i}))\). The former case has no effect on the belief support, and in the latter case the statement holds as long as \(\tok\mysucc((\belsupf(h_{i-1}),\alpha),a_{i-1},\belsupf(h_{i})) \neq \emptyset \). However \(\belsupf(h_i) \) is defined as \(\belsupup(\belsupf(h_{i-1}), a_{i-1},\ob_i,\obc(o_{i-1},\ell_{i-1},\ell_i))\neq \emptyset \), therefore it holds \(\belsupf(h_{i})\in \mysucc((\belsupf(h_{i-1}),a_{i-1}) \), and therefore there exists \((\belsupf(h_{i}),\beta)\in \tok\mysucc((\belsupf(h_{i-1}),\alpha),a_{i-1},\belsupf{(h_i)}) \).
		\end{proof} 
	
	\begin{invariant}\label{app:inv2}
		\(\ell_i'\leq \ell_i \)
	\end{invariant}
	\begin{proof}
 At step \(0 \) this holds. Assume this holds at step \( i-1\). The energy levels change only when \(\pi \) outputs an action \(a_{i-1} \), then it holds \(\ell_i=\resup(\ell_{i-1},s_{i-1},a_{i-1}) \) and \(\ell_i'=\resup(\ell_{i-1}',(\belsupf(h_{i-1}),\alpha_{i-1}),a_{i-1}) \). From \(s_{i-1}\in \belsupf(h_{i-1}) \) and \(\cpomdp \) being consistent, it holds that \(\cons(s_{i-1},a_{i-1})=\tok\cons((\belsupf(h_{i-1}),\alpha_{i-1}),a_{i-1} ) \). Also it holds that \((\belsupf(h_{i-1}),\alpha_{i-1})\in \tok\reloads' \) implies \(\belsupf(h_{i-1})\subseteq \reloads \). Therefore both levels are decreased by the same amount, and if \(\ell_{i-1}' \) is reloaded then also \(\ell_{i-1} \) is reloaded. 
\end{proof}

	\begin{invariant}\label{app:inv3}
	If \(\ell_0\geq  \tlevPR^{\tok\cpomdp'}(\belsupf(h_0)) \) then \(\ell_i' \geq \tlevPR^{\tok\cpomdp'}(\belsupf(h_i)) \)	
	\end{invariant}
	\begin{proof}
		 \(i=0 \) is trivial. Let us assume this holds for \(i-1 \). The only time when either \(\tlevPR^{\tok\cpomdp'}(\belsupf(h_i)) \) or \(\ell_i' \) changes is when \(\pi \) chooses an action \(a_{i-1} \). However as this action is chosen using a safe policy in the simulation, the new state of the simulation is a valid successor according to the used policy, and the energy level is adjusted accordingly as well, this can be seen as a computational step in \(\tok\cpomdp' \) under a safe policy. Therefore there is a safe policy from \(\belsupf(h_i)\ell_i' \). And from Lemma~\ref{app:lemma-safe-produce-positive} we obtain \(\ell_i' \geq \tlevPR^{\tok\cpomdp'}(\belsupf(h_i))\).
	 \end{proof}

	From Invariants~\ref{app:inv2} and~\ref{app:inv3} we have that \(\pi \) is safe. And therefore it holds that as long as \(\pi \) does not reach a goal it must reach a reload infinitely often (or run out of resource). This property holds also inside of the simulation. This means, that if the counter is set to \(0 \), the simulation either reaches a goal, or eventually begins to simulate \(\pi' \) from some reload \((\bels_1,\alpha_1)\in \tok\reloads' \) for some \(\alpha_1\neq \varepsilon \). Since \(\pi' \) is positive-goal, there exists a path \((\bels_1,\alpha_1),\dots,(\bels_m,\alpha_m) \) of length at most \(|\tok\states|\cdot \ca \) under \(\pi' \)  in \(\tok\cpomdp' \) ending in a goal (as in CMDP it is sufficient to consider the current state and resource level, and there are only \(|\tok\states|\cdot \ca \) different configurations). Note that this holds for each \(\alpha\in \bels_1 \). Since \(\alpha_1 \) is chosen uniformly at random, there is a positive probability the current state \(s_1\) of \(\cpomdp \) will be chosen. And then during each of the at most \(|\tok\states|\cdot \ca \) steps there is a positive probability that the next \(\alpha_i \) will be chosen as the next state \(s_i \). Therefore there is a positive probability that \(\pi \) will reach a goal within the \(|\tok\states|\cdot \ca \) steps, and in the case this does not happen, the whole process repeats, thus ensuring that eventually the goal will be reached with probability 1. 
	Therefore \(\ell_0\geq  \tlevPR^{\tok\cpomdp'}(\belsupf(h_0)) \) implies that \(\pi \) is a safe goal policy for the history \(h=h'\ob_0\ell_0 \), and therefore \(\tlev^{\cpomdp}(h)\leq \tlevPR^{\tok\cpomdp'}(\belsupf(h)) \).
\end{proof}

\section{Proof of Theorem~\ref{thm:fipomdp}}

\begin{theorem*}[\textbf{3}]
Let \( N \) be the decision horizon and consider a finite horizon approximation of the RSGO problem where the costs are accumulated only over the first \( N \) steps. 
Consider any decision step of FiPOMDP and let \( h \) be the history represented by the current root node of the search tree. Let \( p_h \) be the probability that the action selected by POMCP to be played by the agent is an action used in \( h \) by an optimal finite-horizon safe goal policy, and \( \mathit{sim} \) the number of simulations used by FiPOMDP. Then for \( \mathit{sim} \rightarrow \infty \) we have that \( p_h \rightarrow 1 \).
\end{theorem*}

This follows directly from Theorem~1 in~\cite{SV:2010:POMCP}, since the behaviour of FiPOMDP mimics the behaviour of POMCP on a CoPOMDP whose history tree has been pruned by the computed shield.

\section{Benchmarks}
\label{app-sec:benchmarks}
The results in Table~1 were averaged over 100 runs of each experiment.

\subsection*{Resource-Constrained Tiger}
The first toy benchmark is a \emph{resource-constrained Tiger,} a modification of the classical \emph{Tiger} benchmark for POMDPs~\cite{kaelbling1998planning} adapted from~\cite{DBLP:journals/corr/BrazdilCCGN16}, where the goal states represent the situation where the agent has made a guess about the tiger's whereabouts. 

In the classical tiger, the agent is standing in front of two doors. Behind one door, there is a treasure, behind the other is a tiger; the agent does not now the position of the treasure/tiger. Still, he has an option to open one of the doors. Opening a door with a treasure yields a reward, while upon finding the tiger, the animal eats the poor agent, yielding a large negative penalty. Luckily, the agent has another action at its disposal: listening for the tiger's roar and movements. Upon playing the listening action, the agent receives one of the observations \emph{left door,} \emph{right door,} indicating the position of the tiger. The problem is that the observation is noisy: with some probability \( p \), the agent mishears and tiger is actually behind the other door than suggested by the observation. Hence, the agent has to repeat the listening sufficiently often to build enough confidence in the tiger's position.

In the resource-constrained variant, the agent's listening actions consume a single unit of energy, necessitating regular reloads (agent's total capacity is 10). During each reload, there is a \( 0.2 \) probability that the tiger switches its position with the treasure, which induces an interesting tradeoff into determining the optimal time in which the agent should make the guess. There is a cost of 10 per each step, opening door with tiger/treasure yields cost 5000/-500. We consider two versions: \emph{simple}, where the probability  \(p \) of an observation being correct is \( 0.85 \), and \emph{fuzzy}, where this probability is decreased to \( 0.6 \).

\paragraph{Tiger Hyperparameters.} For this benchmark, we use exploration constant 1, decision horizon 500, and 100 POMCP simulations per decision.

\subsection*{UUV Gridworld}

The second benchmark is a partially observable extension of the \emph{unmanned underwater vehicle (UUV)} benchmark from~\cite{BlaCNOTT:2021:FiMDP}. Here, the agent operates in a grid-world, with actions corresponding to movements in the cardinal directions. Each action has two versions: a \emph{strong} one, which guarantees that the agent moves in the required direction at the expense of larger consumption (2 units of the resource); and a \emph{weak} one, which is cheaper in terms of consumption (1 unit of resource) but under which the agent might drift sideways from the chosen direction due to ocean currents (the exact computation of the dynamics is described in the FiMDPEnv module presented in~\cite{BlaCNOTT:2021:FiMDP}. Partial observability is represented by a noisy position sensor: when the agent visits some cell of the grid, the observed position is sampled randomly from cells in the von Neumann neighbourhood of the true cell. We consider 4 gridworld sizes ranging from 8x8 to 20x20. There is a cost of 1 per step, hitting a goal yields ``cost'' -1000.

The hyperparameters for UUV Gridworld benchmarks are provided in Table~\ref{tab:hyperUUV}.

\begin{table*}[t!]
\centering
\begin{tabular}{lcccc}
\toprule
Benchmark & Exploration constant & Decision horizon & Simulations per step & Capacity\\
\midrule
UUV grid 8x8 & 25 & 100 & 1000 & 12 \\
UUV grid 12x12 & 50 & 100 & 1000 & 16 \\
UUV grid 16x16 & 200 & 100 & 1000 & 16 \\
UUV grid 20x20 & 200 & 100 & 1000 & 20
\end{tabular}
\caption{Hyperparameter settings ofr the UUV Gridworld benchmarks.}
\label{tab:hyperUUV}
\end{table*}

\subsection*{Manhattan AEV}
The description of this benchmark is provided in the main text of the paper. We use capacity 1000, decision horizon 1000, 100 simulations per step, and exploration constant 1.


\section{Reproducibility Checklist Clarifications}

\begin{itemize}
\item \textbf{Reproducibility checklist (4.3)}
\emph{Proofs of all novel claims are included.} Yes:
We include proofs of the lemmas and theorems in this technical appendix. Theorem~1 was proved in the previous work, which we cite.

\item \textbf{Reproducibility checklist (4.6)}
\emph{All theoretical claims are demonstrated empirically to hold.} Yes: Our main theoretical claim, implicitly contained Theorem 2, is that our algorithm computes a shield that prevents resource exhaustion. Indeed, in our experiments, the shielded agent never exhausted the resource.

\item\textbf{Reproducibility checklist (4.7)}
\emph{All experimental code used to eliminate or disprove claims is included.}
Not applicable: we do not use code to disprove or eliminate any claims.

\item \textbf{Reproducibility checklist (5.1)}
\emph{A motivation is given for why the experiments are conducted on the selected datasets.}
Partial: The (modified) Tiger benchmark was used because Tiger is a well-known fundamental benchmark for POMDPs. The UUV and Manhattan benchmarks was shown because they were used in the previous work on resource-constraint agent planning (albeit in perfectly observable setting), and our modifications incorporate observability in a natural way.

\item \textbf{Reproducibility checklist (5.2)}
\emph{All novel datasets introduced in this paper are included in a data appendix.}
Yes: Our data appendix contains Python modules that generate internal (i.e., used by our tool) representations of the POMDPs we experimented with.

\item \textbf{Reproducibility checklist (5.4)}
\emph{All datasets drawn from the existing literature (potentially including authors’ own previously published work) are accompanied by appropriate citations.}
Not applicable: While our benchmarks are inspired by benchmarks from the previous work (and we clearly reference that previous work), we employ their custom modifications (since the original versions do not consider either partial observability, or resource constraints). Hence, our benchmarks are not really ``drawn'' from the existing literature.

\item \textbf{Reproducibility checklist (6.2)}
\emph{All source code required for conducting and analyzing the experiments is included in a code appendix.}
Yes: we include an anonymized snapshot of our working repository together with a list of software requirements. The file \texttt{REPROD-GUIDE.md} in the root of our repository contains instructions on how to replicate our experiments.

\item \textbf{Reproducibility checklist (6.4)}
\emph{All source code implementing new methods have comments detailing the implementation, with references to the paper where each step comes from.}
Partial: The classes and methods in our implementation are documented with Python docstrings. However, we do not provide detailed references to the paper, since the current version of the paper was created only after the core of our implementation. Still, the overall structure of the implementation corresponds to the description in the paper and should thus be navigable.

\item \textbf{Reproducibility checklist (6.7)}
\emph{This paper formally describes evaluation metrics used and explains the motivation for choosing these metrics.}
Yes: the metrics used in Table~\ref{tab:results} (survival and reachability probabilities, expected costs) are standard and natural given the problem we are tackling.

\item \textbf{Reproducibility checklist (6.8)}
\emph{This paper states the number of algorithm runs used to compute each reported result.}
Yes: 100 runs per experiment, as stated at the beginning of Section~\ref{app-sec:benchmarks}. The section also contains description of hyperparameter values for the individual experiments.

\item \textbf{Reproducibility checklist (6.12)}
\emph{This paper states the number and range of values tried per (hyper-) parameter during development of the paper, along with the criterion used for selecting the final parameter setting.}
Partial: This pertains to the setting of the exploration constant hyperparameter. All experiments were tested with exploration constant 1 (since we use normalized costs in the UCT formulas). We briefly manually experimented with other values. For Tiger and Manhattan benchmarks, this did not elicit significant change. For the UUV benchmarks, we observed improvements in hit ratios when the exploration constant is increased: this is likely because the UUV benchmarks are rather structureless and a lot of exploration is needed to hit a goal (though too large exploration is, unsurprisingly, also damaging). We manually tried several values from the 1-200 range, picking the best in terms of hit ratio. We note that this manual tuning pertains to the POMCP part of our algorithm, and not to shielding, which is the main conceptual novelty of our paper. Also for this reason we did not employ any systematic hyperparameter tuning.
\end{itemize}

\end{document}